\documentclass[10pt,journal,compsoc]{IEEEtran}
\usepackage{amsfonts,amsmath,amssymb,amsthm,bm}
\usepackage{booktabs,threeparttable}
\usepackage{graphicx,psfrag,epsf}
\usepackage{enumerate}
\usepackage{url}
\usepackage{algorithm}
\usepackage{algpseudocode}
\usepackage{helvet}
\usepackage[colorlinks=true,pagebackref,linkcolor=magenta]{hyperref}
\usepackage{color} 
\usepackage{lineno}
\usepackage{chngcntr}
\usepackage{thmtools, thm-restate}

\theoremstyle{plain}

\newtheorem{definition}{Definition}

\usepackage{adjustbox}
\usepackage{dsfont}
\ifCLASSOPTIONcompsoc
  \usepackage[nocompress]{cite}
\else
  \usepackage{cite}
\fi

\ifCLASSINFOpdf

\else
 
\fi

\begin{document}
  \title{\bf Decision Tree Embedding by Leaf-Means}
  \author{Cencheng Shen, Yuexiao Dong, Carey E. Priebe
  \IEEEcompsocitemizethanks{
  \IEEEcompsocthanksitem Cencheng Shen is 
  with the Department of Applied Economics and Statistics, University of Delaware,
  E-mail: shenc@udel.edu \protect
   \IEEEcompsocthanksitem Yuexiao Dong is with the Department of Statistics, Operations, and Data Science, Temple University. E-mail: ydong@temple.edu \protect
  \IEEEcompsocthanksitem Carey E. Priebe
is with the Department of Applied Mathematics and Statistics (AMS), the Center for Imaging Science (CIS), and the Mathematical Institute for Data Science (MINDS), Johns Hopkins University. E-mail: cep@jhu.edu \protect
}
}

\IEEEtitleabstractindextext{%
\begin{abstract}
Decision trees and random forest remain highly competitive for classification on medium-sized, standard datasets due to their robustness, minimal preprocessing requirements, and interpretability. However, a single tree suffers from high estimation variance, while large ensembles reduce this variance at the cost of substantial computational overhead and diminished interpretability. In this paper, we propose Decision Tree Embedding (DTE), a fast and effective method that leverages the leaf partitions of a trained classification tree to construct an interpretable feature representation. By using the sample means within each leaf region as anchor points, DTE maps inputs into an embedding space defined by the tree’s partition structure, effectively circumventing the high variance inherent in decision-tree splitting rules. We further introduce an ensemble extension based on additional bootstrap trees, and pair the resulting embedding with linear discriminant analysis for classification. We establish several population-level theoretical properties of DTE, including its preservation of conditional density under mild conditions and a characterization of the resulting classification error. Empirical studies on synthetic and real datasets demonstrate that DTE strikes a strong balance between accuracy and computational efficiency, outperforming or matching random forest and shallow neural networks while requiring only a fraction of their training time in most cases. Overall, the proposed DTE method can be viewed either as a scalable decision tree classifier that improves upon standard split rules, or as a neural network model whose weights are learned from tree-derived anchor points, achieving an intriguing integration of both paradigms.
\end{abstract}

}

\maketitle
\IEEEdisplaynontitleabstractindextext
\IEEEpeerreviewmaketitle

\IEEEraisesectionheading{\section{Introduction}}

Decision trees are among the most interpretable and foundational models in supervised machine learning. A decision tree recursively partitions the feature space using axis-aligned splits to minimize an impurity measure, such as Gini impurity or mean squared error, producing a hierarchy of rules that map inputs to class labels or responses \cite{quinlan1993c45,breiman1984cart}. Despite their intuitive appeal and interpretability, a single decision tree often suffer from high variance and overfitting.

An ensemble of decision trees, most notably random forest (RF) \cite{breiman1996bagging, breiman2001random}, constructs a collection of trees, each trained on a bootstrap sample of the data and using a random subset of features at each split. The final prediction is obtained by averaging (for regression) or majority voting (for classification) across all trees. This ensemble strategy substantially reduces variance while maintaining low bias, yielding a strong, general-purpose learner that often achieves much better predictive accuracy than a single decision tree and performs well across a wide range of tasks. Since then, numerous works have sought to improve random forest, including approaches based on oblique splits, sparsity, or rotation-based transformations \cite{RFSparse, rodriguez2006rotation, menze2011oblique}.

On the other hand, neural networks (NN) have been the foundational component behind recent progress on complex data regimes such as text, images, and videos. They are composed of multiple layers of affine transformations and nonlinear activations, often supplemented with techniques such as dropout and skip connections \cite{srivastava2014dropout, he2016deep} to facilitate effective training. Early work established that feed-forward networks are universal function approximators \cite{hornik1991approximation}. With the advent of efficient backpropagation and GPU computing, deep neural networks have become highly scalable to large datasets far beyond the capacity of many traditional machine learning methods, and this scalability has driven state-of-the-art results across vision, speech, and language domains \cite{lecun2015deep,goodfellow2016deep}: Convolutional neural networks emerged as the dominant architecture for image processing tasks \cite{lecun1998gradient}, while more recent transformer-based architectures employing self-attention mechanisms \cite{vaswani2017attention} have revolutionized sequence modeling and natural language processing. Deep learning thus represents the prevailing framework in modern machine learning, providing significant advantages in scalability to big data and in inference speed compared to other machine learning families.

Nevertheless, recent empirical studies have shown that random forest and other tree-based ensemble methods outperform neural networks in many standard data settings \cite{caruana2006empirical,fernandez2014we,grinsztajn2022why}, particularly those involving mixtures of numerical and categorical variables with heterogeneous scales. Random forest require minimal preprocessing, and often work well out of the box without extensive effort devoted to complex architectures, hyperparameter tuning, or regularization. As a result, they tend to generalize better on moderate-sized standard datasets, where deep networks are more susceptible to overfitting or undertraining. 

Despite its versatility, random forest suffers from significant computational bottlenecks \cite{ferreira2021computational}. First, training hundreds of trees can be computationally expensive for very large datasets or high-dimensional feature spaces. Second, when individual trees grow deep, inference can become considerably more costly. Third, the memory cost can be substantial for large ensembles, and inference latency scales with the number of trees. In addition, although a single decision tree is interpretable, the ensemble as a whole can behave like a black box.

There also exists many works exploring hybrid strategies combining random forest with other methods, such as logistic regression \cite{He2014}, as well as abundant explorations in various ways to incorporate neuron networks \cite{Kontschieder2015,wang2016rf2nn, biau2019neural,togunwa2023deep,konstantinov2023attention, qiu2024nnrf,nickzamir2025hybrid}. However, most hybrid RF–NN models remain ad hoc and application-driven, with limited theoretical understanding. Computationally, such hybrids almost always incur increased training time and memory usage, as they often combine the ensemble complexity of forest with the parameter-rich nature of neural networks. 

Given a dataset $(\mathbf{X}, \mathbf{Y})$, where $\mathbf{X} \in \mathbb{R}^{n \times p}$ represents $n$ training samples with $p$ features and $\mathbf{Y} \in \{1,\dots,K\}^n$ denotes the corresponding class labels, we propose the decision tree embedding (DTE) method based on a single classification tree. The central idea is to use the sample means of each leaf region as anchor points that transform the input into a new embedding space, yielding $\mathbf{Z} \in \mathbb{R}^{n \times m}$. Although the axis-aligned splitting rules of a single decision tree have high variance and thereby degrade inference quality, our key insight is that the sample means within the leaf regions have much lower variance by its aggregating nature. Consequently, these anchor points extracted from a single tree yield a stable and informative embedding that can be readily used for subsequent classification. The resulting formulation can be viewed either as a scalable tree-based classifier, or as a neural network model whose weights are learned from decision trees.

Section~\ref{sec:method} presents the main method, along with an ensemble extension using additional bootstrap trees, a subsequent classification step based on linear discriminant analysis, a computational complexity analysis, and qualitative comparisons with random forest and classification neural networks. Section \ref{sec:theory} discusses the population-level theoretical properties of DTE, including its preservation of conditional density under mild conditions and a characterization of the resulting classification error. Section~\ref{sec:sim} visualizes the intrinsic mechanism of DTE using synthetic data, and Section~\ref{sec:real} shows that DTE achieves an excellent balance between accuracy and running time, comparing favorably with decision trees, random forest, and shallow neural networks. 

\section{Method}
\label{sec:method}
\subsection{Decision Tree Embedding using A Single Tree}

Our main method proceeds as follows:
\begin{itemize}
\item \textbf{Input}: A labeled dataset $(\mathbf{X}, \mathbf{Y}) \in \mathbb{R}^{n \times p} \times \{1,\dots,K\}^n$.
\item \textbf{Step 1}: Train a standard classification tree that partitions the feature space $\mathbb{R}^p$ into disjoint regions $\{\mathcal{R}_1, \mathcal{R}_2, \dots, \mathcal{R}_m\}$, each corresponding to a leaf region. Let $S_j = \{i : \mathbf{X}_i \in \mathcal{R}_j\}$ denote the set of sample indices that fall into region $\mathcal{R}_j$.
\item \textbf{Step 2}: Compute the mean for each leaf region:
\begin{align*}
    \mu_j = \frac{1}{|S_j|} \sum_{i \in S_j} \mathbf{X}(i,:), \quad j = 1, \dots, m.
\end{align*}
Collect these leaf means into a matrix $\mathbf{W} = [\mu_1; \mu_2; \dots; \mu_m] \in \mathbb{R}^{m \times p}$.
\item \textbf{Step 3}: Construct the embedding $\mathbf{Z}$ as 
\begin{align*}
    \mathbf{Z} = \mathbf{X} \mathbf{W}^\top + \mathbf{1} \mathbf{b}^\top \in \mathbb{R}^{n \times m},
\end{align*}
where $\mathbf{1}$ is an $n\times 1$ vector of ones, and $\mathbf{b}$ is an $m \times 1$ intercept vector defined as $\mathbf{b}(j)=|\mathbf{W}(j,:)|^2$ for $j=1,\ldots,m$.
\item \textbf{Output}: Final embedding $\mathbf{Z}$, weight matrix $\mathbf{W}$, and intercept vector $\mathbf{b}$.
\end{itemize}
Intuitively, each column of $\mathbf{Z}$ measures the affinity of samples to a specific leaf region, determined by their inner product with the corresponding leaf mean. 
Each leaf mean serves as a data-driven anchor point summarizing a local region, and the embedding dimension equals the number of leaves. We refer to this single-tree version as $\text{DTE}_1$.

\subsection{Decision Tree Embedding with Additional Bootstrap Trees}

The single-tree embedding naturally extends to an ensemble variant, denoted $\text{DTE}_t$ for an ensemble of $t$ trees. The first tree is always trained on the original dataset, while the remaining $t-1$ trees are trained on independent bootstrap samples. Each tree $s = 1, \dots, t$ yields its own leaf-mean matrix $\mathbf{W}^{(s)}$ and intercept vector $\mathbf{b}^{(s)}$, which are concatenated to form the final embedding.

\begin{itemize}
    \item \textbf{Input}: A labeled dataset $(\mathbf{X}, \mathbf{Y})$, and a positive integer $t$.
\item \textbf{Step 1}: Compute the single-tree embedding at $s=1$:
    \begin{align*}
        [\mathbf{Z}^{(1)},\mathbf{W}^{(1)}, \mathbf{b}^{(1)}] = \text{DTE}_1(\mathbf{X}, \mathbf{Y}).
    \end{align*}
\item \textbf{Step 2}: For $s=2,\ldots,t$, bootstrap the data to obtain $(\mathbf{X}^{(s)}, \mathbf{Y}^{(s)})$, and compute
\begin{align*}
        [\mathbf{Z}^{(s)},\mathbf{W}^{(s)}, \mathbf{b}^{(s)}] = \text{DTE}_1(\mathbf{X}_{trn}^{(s)}, \mathbf{Y}_{trn}^{(s)}).
    \end{align*}
\item \textbf{Step 3}: Concatenate the results:
\begin{align*}
    \mathbf{W} = \begin{bmatrix} \mathbf{W}^{(1)} \\[4pt] \mathbf{W}^{(2)} \\[4pt] \vdots \\[4pt] \mathbf{W}^{(t)} \end{bmatrix}, 
    \quad
    \mathbf{b} = [\mathbf{b}^{(1)}, \mathbf{b}^{(2)},\ldots, \mathbf{b}^{(t)}],
\end{align*}
\begin{align*}
    \mathbf{Z}^ = [\mathbf{Z}^{(1)}, \mathbf{Z}^{(2)},\ldots, \mathbf{Z}^{(t)}],
\end{align*}
\item \textbf{Output}: Final embedding $\mathbf{Z}$, weight matrix $\mathbf{W}$, and intercept $\mathbf{b}$.
\end{itemize}
Suppose the $s$th tree has $m^{(s)}$ leaf regions. Then the dimension of $\mathbf{Z}$ from $\text{DTE}_t$ equals the total number of leaf regions across all trees, i.e., $m = \sum_{s=1}^{t} m^{(t)}$. While alternative aggregation schemes for combining multiple embeddings are possible, direct concatenation is both simple and information-preserving.

\subsection{Classification by LDA}

The decision tree embedding $\mathbf{Z}$, whether constructed from a single tree or an ensemble, can be seamlessly paired with a lightweight linear classifier such as linear discriminant analysis (LDA) for downstream prediction tasks. Given a train/test split, the procedure is as follows:

\begin{itemize}
    \item \textbf{Input}: Training data $(\mathbf{X}_{trn}, \mathbf{Y}_{trn})$, testing data $\mathbf{X}_{tsn}$, and the number of trees $t$.
    \item \textbf{Training}: Compute the training embedding:
    \begin{align*}
        [\mathbf{Z}_{trn},\mathbf{W}, \mathbf{b}] = \text{DTE}_t(\mathbf{X}_{trn}, \mathbf{Y}_{trn}).
    \end{align*}
    \item \textbf{Testing Data Embedding}: Project the test data into the embedded space:
    \begin{align*}
        \mathbf{Z}_{tsn} = \mathbf{X}_{tsn} \mathbf{W}^\top + \mathbf{1} \mathbf{b}^\top.
    \end{align*}
    \item \textbf{Testing}: Train an LDA classifier on $(\mathbf{Z}_{trn}, \mathbf{Y}_{trn})$, then apply it to $\mathbf{Z}_{tsn}$.
    \item \textbf{Output}: Predicted test labels $\hat{\mathbf{Y}}_{tsn}$.
\end{itemize}

In cross-validation, the predicted labels are compared to ground truth to compute classification accuracy. While other classifiers, such as logistic regression, shallow neural networks, or nearest-neighbor methods, can also be used atop the decision tree embedding, LDA offers an excellent balance of computational efficiency and accuracy.

\subsection{Computational Complexity}

Let $n$ be the number of samples, $p$ the feature dimension, $m$ the number of leaf regions, $t$ the number of trees, and $K$ the number of classes in the downstream classifier.
\begin{itemize}
\item Training $t$ classification trees has time complexity $O(t n p \log n)$; when done in parallel, which is feasible since DTE typically uses only a small $t$, the cost reduces to $O(n p \log n)$.
\item Computing all leaf means requires aggregating each sample once, for a total cost of $O(n p)$.
\item The matrix multiplication $\mathbf{Z} = \mathbf{X} \mathbf{W}^\top$ costs $O(n p m)$.
\item LDA costs $O(n m^2 + m^3)$ for training, and per-sample testing time is $O(m^2+K*m)$.
\end{itemize}
Since $p$, $m$, and $K$ are typically much smaller than $n$, the overall computational complexity is effectively linear in $n$ for large sample sizes. Compared with random forest, the decision tree embedding requires only a few trees (we found $t=1$ or $t=3$ to be sufficient), making it substantially faster and more memory-efficient, as most of the cost arises from tree construction.

In practice, a few additional considerations may affect efficiency. For example, if a single tree produces too many leaf regions (large $m$), the classification step by LDA can slow down. This can be mitigated by limiting the tree depth or the number of splits, although we did not apply such constraints in this paper. 

\subsection{Comparison to Tree and Forest Classification}
The DTE method requires building decision trees, but the subsequent classification step fundamentally differs from the splitting rules used in tree-based classifiers. In standard decision trees, small perturbations in the data can lead to large deviations in the resulting split. Consequently, a single tree often exhibits high variance. Random forest mitigates this by aggregating many such unstable trees, thereby reducing variance through averaging.

In contrast, DTE does not use the split rules at all for subsequent classification. Instead, it uses the sample means of the leaf regions, which has lower variance than the split rules because of sample averaging in the leaf region.

An illustrative example is the following. Suppose a random variable $X$ satisfies $X|Y=1 \in [-1,0]$, and $X|Y=2 \in [0,1]$, uniformly distributed in each respective interval. The perfect split is clearly at $0$. A decision tree using limited training samples may place the split at, say, $0.04$, which is suboptimal for testing data from the same distribution. A random forest would reduce the deviation: different bootstrap samples may produce splits around $0.02$, $-0.02$, $-0.03$, etc., and averaging these yields a split location closer to $0$ than any single tree.

Now consider the leaf means. With the perfect split at $0$, the ideal anchor points are simply the true class means, $[-0.5, 0.5]$. Using the imperfect split at $0.04$ instead yields empirical leaf means at $[-0.48, 0.52]$. Both sets of anchor points linearly separate the two classes, and the discrepancy between them is smaller than the discrepancy between split locations themselves. This illustrates why the embedding is robust against sampling noise with just a single tree.

Therefore, DTE can be viewed as a fast and efficient tree-based classifier that inherently reduces the variance arising from unstable split rules. The variance reduction comes from averaging samples within each leaf region, whereas random forest reduces variance by aggregating many diverse trees, a fundamentally different and more computationally expensive mechanism. Random forest typically requires many trees to stabilize performance, whereas DTE often achieves similar performance with just one or a few trees. In fact, using too many trees can offset the benefits of the embedding during LDA classification. As $t$ increases, the embedding dimension $m$ grows proportionally, leading to higher estimation variance and increased running time for LDA. 

\subsection{Relationship to Classification Neural Networks}
Once a decision tree is built, DTE becomes a purely linear model. In our proposed classification pipeline, DTE provides a linear embedding with an intercept, followed by LDA, which is another linear transformation with an intercept. Consequently, the overall model is equivalent to a shallow two-layer neural network: DTE serves as the hidden layer with $m$ neurons, and LDA acts as the output layer. In particular, the leaf regions and their associated anchor points can be viewed as a sufficient transformation in the sense of \cite{NNGraphSufficiency}.

Because model inference in DTE with LDA involves only basic matrix operations similar to a trained neural network, it inherently avoids the potentially expensive tree traversal process used in random forest, where inference time increases with both tree depth and the number of trees.

Therefore, DTE can be viewed as a neural network model, except that the number of neurons $m$ and the weight matrices are learned from the structure of a decision tree, rather than from random initialization followed by back-propagation. It is often the case that constructing a single classification tree is significantly faster than running gradient-based training. While this does not guarantee superiority over back-propagation, when the decision tree provides a good partition of the feature space, DTE is expected to perform favorably. 

\section{Theory}
\label{sec:theory}
While the previous sections considered the sample-based version of the method, working with sample data $\mathbf{X}$ and the embedding $\mathbf{Z}$, in this section we study the corresponding random-variable formulation. Specifically, we analyze the population-level behavior of the method in terms of random variables $X$ and $Y$. Under the classical i.i.d.\ assumption, each sample pair $(\mathbf{X}(i,:), \mathbf{Y}(i))$ is drawn independently from the joint distribution of $(X,Y)$ for $i = 1,\ldots,n$. 

Our focus is exclusively on classification, so $X$ takes values in $\mathbb{R}^{p}$ and $Y$ is a discrete label in $\{1,\ldots,K\}$. We denote their supports by $\mathcal{X}$ and $[K]$, respectively.

\subsection{Bayes-homogeneous Partition and Conditional Probability Preservation}

We first establish that the decision tree embedding preserves the conditional class distribution whenever the tree’s leaf regions are Bayes-homogeneous, defined as follows.

\begin{definition}
Given random variables $(X,Y) \in \mathcal{X} \times [K]$, a set of regions $\{\mathcal{R}_1,\ldots,\mathcal{R}_m\}$ partitioning the support $\mathcal{X}$ is said to be Bayes-homogeneous, if the conditional distribution of $Y$ given $X$ is constant within each region. That is, for every region 
$\mathcal{R}_j$ and all pairs of points $x, x' \in \mathcal{R}_j$
\[
P(Y \mid X = x) = P(Y \mid X = x')
\]
The set of regions is said to be $\epsilon$-Bayes-homogeneous if, their conditional distributions satisfy
\[
\|P(Y \mid X = x) - P(Y \mid X = x')\|_1 \le \epsilon,
\]
where $\epsilon$ is the maximum deviation across all regions.
\end{definition}
This definition is flexible enough to allow arbitrary $\epsilon$, making it universally applicable. In practice, a well-behaved decision tree is expected to yield a set of leaf regions with a small $\epsilon$. Our first theorem shows that the decision tree embedding approximately preserves the conditional class distribution up to this same bound, and the preservation becomes exact in the ideal case $\epsilon = 0$.

Note that we use the notation $P(Y \mid X)$ to denote the entire conditional distribution of $Y$ given $X$, i.e.,
\[
P(Y \mid X)
= \big(P(Y=1\mid X),\,\ldots,\,P(Y=K\mid X)\big),
\]
a vector in the probability simplex. When referring to a specific class $c$, we explicitly write $P(Y=c \mid X)$.

\begin{restatable}{theorem}{thmOne}
\label{thmOne}
Let $(X,Y)$ be a random pair with $X\in\mathbb{R}^p$ and $Y\in [K]$. Let $\{\mathcal{R}_1,\ldots,\mathcal{R}_m\}$ be the set of leaf regions from a classification tree, and let $Z=X\mathbf{W}^\top+\mathbf{1}\mathbf{b}^\top$ denote its Decision Tree Embedding, where
\[
\mu_j = \mathbb{E}[X\mid X\in\mathcal{R}_j], \quad
\mathbf{b}_j = -\|\mu_j\|^2, \quad
\mathbf{W} = [\mu_1^\top,\ldots,\mu_m^\top]^\top.
\]
If the partition $\{\mathcal{R}_j\}$ is $\varepsilon$-Bayes-homogeneous, then
\[
\|P(Y\,|\,X)-P(Y\,|\,Z)\|_1 \le \varepsilon.
\]
In particular, when $\varepsilon=0$, the DTE embedding is sufficient for $Y$; that is,
\[
P(Y\,|\,X)=P(Y\,|\,Z) \quad \text{a.s.}
\]
\end{restatable}

Therefore, DTE is able to approximately preserve the conditional label probabilities. Each leaf region $\mathcal{R}_j$ defines an anchor point $\mu_j$, whose corresponding coordinate in $Z$ serves as a similarity-based weight. In the ideal case, DTE is fully sufficient, providing a lossless transformation of the underlying class-conditional structure.

\subsection{Population Classification Error Bound}

Our next theorem derives a classification error bound.

\begin{restatable}{theorem}{thmTwo}
\label{thmTwo}
Under the same notation as Theorem~\ref{thmOne}, and further assume that each point in a leaf region $\mathcal{R}_j$ is closer (in Euclidean norm) to its own region mean $\mu_j$ than to any other region mean $\mu_{j'}$.

For each class $c \in [K]$, let 
\begin{align*}
\mathcal{J}_c &= \{j:\arg\max_{c'}P(Y=c' \mid X\in\mathcal{R}_j)=c\}
\end{align*}
be the index set of leaves whose majority label is $c$. We then consider the indicator classifier on the DTE embedding:
\[
g(Z) = \arg\max_{c\in\{1,\ldots,C\}} \gamma_c^\top Z, \quad
(\gamma_c)_j = 
\begin{cases}
1, & j\in \mathcal{J}_c,\\
0, & \text{otherwise.}
\end{cases}
\]

For each $j$, let
\begin{align*}
l_j &= 1-\max_c P(Y=c \mid X\in\mathcal{R}_j)
\end{align*}
be the impurity level at each region. Then the classification error of the indicator rule $g$ equals 
\[
L_{g} = \sum_{j=1}^m P(X\in\mathcal{R}_j)\,l_j.
\]

In particular, if every leaf region is pure, we have $l_j=0$ and $L_{g}=0$, such that the DTE embedding achieves perfect classification with the indicator rule.
\end{restatable}

Therefore, when the decision tree performs well so that the resulting leaf partition $\{\mathcal{R}_1,\ldots,\mathcal{R}_m\}$ is pure and the regions are sufficiently well-clustered, meaning each point is closer to its own leaf mean than to any other leaf mean, we have $P(Y=c \mid X\in\mathcal{R}_j)=1$ for some class $c$.  In this case, the classification induced by the DTE embedding is also perfect.

We remark that the indicator classification rule considered in Theorem~\ref{thmTwo} is not necessarily optimal; rather, it is chosen for analytical simplicity and to yield a conservative error bound. In practice, the linear discriminant analysis classifier used in the sample version of our method is more flexible and typically achieves better performance, as it leverages the empirical covariance structure of $Z$ to produce better decision boundaries in finite-sample settings.

\section{Simulation}
\label{sec:sim}
\begin{figure*}[ht]
	\centering
	\includegraphics[width=0.4\textwidth,trim={0cm 0cm 0cm 0cm},clip]{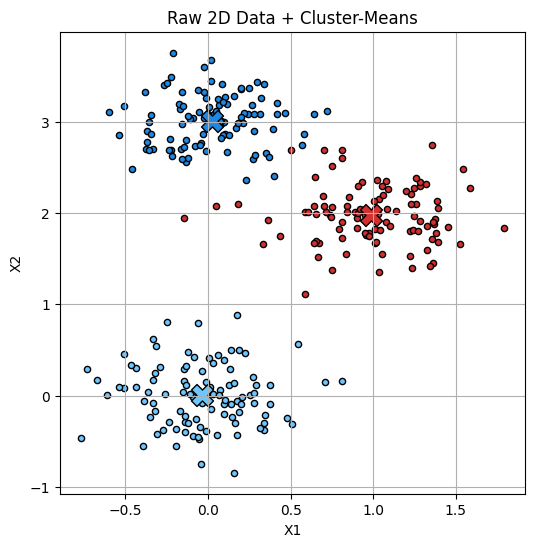}
    \includegraphics[width=0.4\textwidth,trim={0cm 0cm 0cm 0cm},clip]{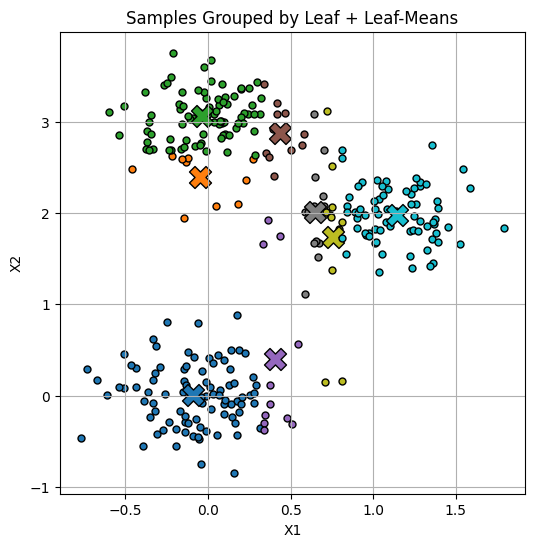}
    \includegraphics[width=0.4\textwidth,trim={0cm 0cm 0cm 0cm},clip]{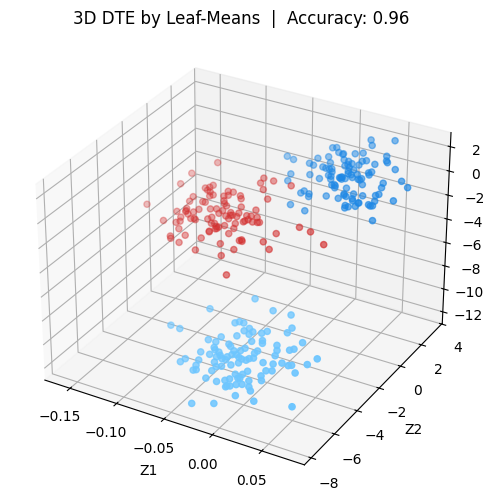}
    \includegraphics[width=0.4\textwidth,trim={0cm 0cm 0cm 0cm},clip]{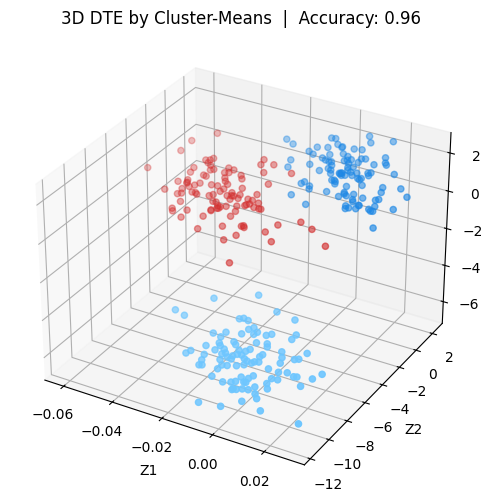}
	\caption{Top left: simulated training data in the original two-dimensional space. The two clusters of class~1 are shown in light and dark blue, respectively, while class~2 is shown in red. Empirical cluster means are indicated by bolded markers in matching colors. Top right: leaf assignments and leaf means of the fitted decision tree. Bottom left: the three-dimensional DTE embedding $\mathbf{Z}$ obtained using the tree-derived leaf means. Samples are colored by their true cluster labels. Bottom right: the oracle embedding $\mathbf{Z}^{*}$, obtained by replacing the leaf means with the true cluster means.}
	\label{fig1}
\end{figure*}

In this section, we consider a simple simulation setting to provide visual intuition for how DTE works. We generate synthetic two-dimensional data drawn from two classes with three total clusters. For class~$1$, the data are generated equally like from either of the following two distributions:
\[
X_{1a} \sim \mathcal{N}\!\left(
\begin{bmatrix}0 \\ 0\end{bmatrix}, 
0.3^2 I_2
\right),
\quad
X_{1b} \sim \mathcal{N}\!\left(
\begin{bmatrix}0 \\ 3\end{bmatrix}, 
0.3^2 I_2
\right).
\]
For class~$2$, the data are generated from a single distribution:
\[
X_{2} \sim \mathcal{N}\!\left(
\begin{bmatrix}1 \\ 2\end{bmatrix}, 
0.3^2 I_2
\right).
\]
The class priors are $\pi_1 = 2/3$ and $\pi_2 = 1/3$. This mixture model has a Bayes-optimal accuracy of approximately $0.993$.

A training set $(\mathbf{X},\mathbf{Y})$ of size $n=100$ is sampled from this model. The top left panel of Figure~\ref{fig1} visualizes the training data: light and dark blue points correspond to the two clusters in class~1, and red points correspond to class~2. The empirical cluster means are shown as well. We then fit a single classification decision tree using a minimum leaf size of 10. The top right panel of Figure~\ref{fig1} colors each sample according to its leaf assignment and marks the corresponding leaf means.

Using this fitted tree, we compute the associated decision tree embedding. The matrix $\mathbf{W}$ contains the leaf means, and the embedding $\mathbf{Z}$ is 8-dimensional because the tree has eight leaf regions. The first three coordinates of $\mathbf{Z}$ are visualized in the bottom left panel of Figure~\ref{fig1}, showing that the two classes are well separated by the resulting embedding.

In this simulation, the ideal separating regions can be viewed as the true generating clusters. Thus, we also construct a hypothetical oracle embedding $\mathbf{Z}^{*}$, obtained by applying the same algorithm but replacing the leaf means with the true cluster means, which serve as optimal anchor points. The resulting embedding is displayed in the bottom right panel of Figure~\ref{fig1}. The close visual agreement between the two bottom panels indicates that the single-tree DTE closely approximates the geometry of the oracle embedding. In particular, although the split rules themselves exhibit high variance relative to the optimal split, the resulting anchor points and the induced embedding exhibit very low variance.

To further assess separation quality, we compute the classification error using DTE with LDA on a test set of size 100 generated from the same model. We perform this both for the actual embedding $\mathbf{Z}$ and the oracle embedding $\mathbf{Z}^{*}$, respectively. Our method yields a remarkably well-separated embedding: an LDA classifier trained on $(\mathbf{Z}, \mathbf{Y})$ achieves an accuracy of $0.96$, which is identical to that obtained from $(\mathbf{Z}^{*}, \mathbf{Y})$.

This shows that the leaf means serve as effective anchor points for a geometry-preserving embedding, and that the resulting decision tree embedding yields high-quality representations even with a single tree. Although the tree-induced partition deviate significantly from the oracle cluster structure, the DTE method remains well-behaved and robust to sampling variability in the tree partitions.

\section{Real Data}
\label{sec:real}

Table~\ref{table1} compares DTE with single decision tree, random forest, and a two-layer neural network on 20 public datasets from the UCI Machine Learning Repository \cite{Kelly2023UCI}, Kaggle, and MATLAB built-ins. The datasets span diverse domains with varying sizes, from text to grayscale images to standard feature data. Categorical features are converted via one-hot encoding.

All experiments were conducted on a standard desktop (6-core Intel i7 CPU, 64GB RAM) using MATLAB 2024a, with no GPU or parallelization. We used various MATLAB built-in functions with default parameters in most cases: For DTE-1, we built a single tree using \texttt{fitctree} (numBins = 30, minLeafSize = 10) and applied LDA using \texttt{fitcdiscr} with pseudolinear discriminant type. DTE-3 uses two additional trees built on independent bootstrap samples. The baseline decision tree uses \texttt{fitctree} with default settings. Random forest use \texttt{TreeBagger} with 50 trees, and the neural network uses \texttt{fitcnet} with 100 hidden neurons and standardization enabled.

Classification errors were averaged over 10 replicates of 5-fold splits (80\% train / 20\% test). The source and size of each dataset are listed in Table~\ref{table1}. The main findings are as follows:
\begin{itemize}
\item DTE-1 vs. decision tree: DTE-1 is equal or better on 19/20 datasets, and significantly better on 16/20.
\item DTE-3 vs. DTE-1: DTE-3 matches or improves upon DTE-1 on all datasets and is significantly better on 6/20.
\item DTE-3 vs. neural network: DTE-3 is better in 13/20 cases, with 9 significant.
\item DTE-3 vs. random forest: DTE-3 is significantly better on 5 datasets, while random forest is significantly better on 6; otherwise performance is comparable.
\end{itemize}
Here, we deem a method significantly better when the difference in error rates exceeds three standard deviations, using the larger standard deviation of the two methods.

Figure~\ref{fig2} reports running times. Across all datasets, DTE-1 and the single decision tree are the fastest methods because they build only one tree. DTE-3 is roughly three times slower (since tree building is repeated without parallelization), but still substantially faster than random forest. Despite MATLAB’s optimized implementation, the 50-tree random forest is about 9× slower than DTE-1 and 4× slower than DTE-3 on median. The neural network is typically faster than the forest but slower than DTE-3 (about 3× on median), with occasional slow convergence. As noted in the complexity analysis, DTE by LDA can be slower when the number of leaf regions $m$ is too large, as seen in the FacePIE dataset, due to the quadratic complexity of LDA in $m$.

Overall, the results highlight the advantages of the proposed method. DTE-1 is consistently more accurate than a single decision tree, even though it uses the same tree. DTE-3 often yields additional gains while remaining efficient. DTE is competitive with random forest and shallow neural networks in accuracy, while offering substantially lower computation time in most cases.

\begin{table*}[ht]
\centering
\begin{tabular}{c|c||cc|ccc}
\hline \hline
Data & $(n,p,K,m)$ & DTE-$1$ & DTE-$3$ & Tree & Forest & S-NN  \\
\hline \hline
Cora \cite{mccallum2000automating} & (2708,1433,7,79) & $27.3 \pm 0.5$ & $25.6 \pm 0.3$ & $34.9 \pm 0.6$ & $\textbf{24.6} \pm 0.3$ & $25.2 \pm 0.5$  \\ \hline
Citerseer \cite{giles1998citeseer} & (3312,3703,6,104)  & $29.9 \pm 0.3$ & $28.7 \pm 0.3$ & $39.9 \pm 0.5$ & $\textbf{27.9} \pm 0.4$ & $28.0 \pm 0.6$ \\ \hline 
Isolet \cite{Fanty1991} & (7797,617,26,140)  & $5.9 \pm 0.1$ & $5.5 \pm 0.1$ & $19.6 \pm 0.3$ & $6.0 \pm 0.1$ & $\textbf{4.4} \pm 0.2$ \\ \hline
FacePIE \cite{SimBakerBsat2003} & (11554,1024,68,681)  & $5.2 \pm 0.1$ & $5.2 \pm 0.1$ & $45.3 \pm 0.5$ & $3.0 \pm 0.1$ & $\textbf{2.9} \pm 0.1$ \\ \hline
FaceYale 32*32 \cite{Fisherfaces} & (165,1024,15,14)  & $30.5 \pm 1.7$ & $\textbf{18.3} \pm 1.3$ & $54.4 \pm 3.0$ & $24.3 \pm 1.4$ & $23.1 \pm 1.8$  \\ \hline
FaceYale 64*64 \cite{Fisherfaces}& (165,4096,15,15)  & $20.8 \pm 1.5$ & $\textbf{10.7} \pm 1.7$ & $37.2 \pm 4.2$ & $19.3 \pm 1.7$ & $18.6 \pm 1.9$ \\ \hline
Iris \cite{fisher1936use} & (150,4,3,3)  & $\textbf{1.9} \pm 0.3$ & $2.1 \pm 0.2$ & $5.2 \pm 1.1$ & $5.0 \pm 0.7$ & $5.5 \pm 1.6$  \\ \hline
Wine \cite{AeberHard1994} & (178,13,3,4)  & $8.2\pm 1.2$ & $2.4 \pm 0.7$ & $10.3 \pm 1.2$ & $\textbf{2.0} \pm 0.6$ & $2.4 \pm 0.6$  \\ \hline
Wisc Cancer \cite{street1993nuclear} & (699,9,2,6)  & $4.2 \pm 0.2$ & $4.2 \pm 0.1$ & $6.0 \pm 0.5$ & $\textbf{3.3} \pm 0.3$ & $4.9 \pm 0.5$ \\ \hline
Colon \cite{alon1999broad} & (62,2000,2,3)  & $13.0 \pm 1.4$ & $\textbf{12.0} \pm 1.0$ & $27.1 \pm 3.9$ & $18.9 \pm 4.5$ & $21.4 \pm 4.2$ \\ \hline
Golub \cite{golub1999molecular} & (72,7129,3,3)  & $8.7 \pm 2.4$ & $6.6 \pm 1.8$ & $9.9 \pm 2.3$ & $\textbf{6.5} \pm 1.8$ & $9.9 \pm 4.0$  \\ \hline
Credit Card \cite{DalPozzolo2018CreditCardFraud} & (284807,30,2,16)  & $0.06 \pm 0.00$ & $0.06 \pm 0.00$ & $0.08 \pm 0.00$ & $\textbf{0.05} \pm 0.00$ & $0.07 \pm 0.00$ \\ \hline
Bank Marketing \cite{moro2011bankmarketing} & (45211,51,2,471)  & $10.0 \pm 0.0$ & $10.0 \pm 0.0$ & $11.6 \pm 0.1$ & $\textbf{9.5} \pm 0.1$ & $11.9 \pm 0.2$ \\ \hline
Student Success \cite{martins2021early} & (4424,36,3,120)  & $24.1 \pm 0.1$ & $24.1 \pm 0.1$ & $30.3 \pm 0.6$ & $\textbf{23.0} \pm 0.4$ & $30.6 \pm 0.4$ \\ \hline
Adult Income \cite{kohavi-nbtree} & (32561,108,2,483)  & $16.6 \pm 0.1$ & $16.6 \pm 0.1$ & $17.2 \pm 0.1$ & $\textbf{13.5} \pm 0.1$ & $18.8 \pm 0.1$ \\ \hline
German Credit & (1000,24,2,30)  & $23.6 \pm 0.4$ & $\textbf{23.4} \pm 0.4$ & $30.4 \pm 1.2$ & $23.5 \pm 0.6$ & $29.6 \pm 0.8$ \\ \hline
Abalone & (4177,10,2,94)  & $22.4 \pm 0.1$ & $22.4 \pm 0.1$ & $26.8 \pm 0.5$ & $\textbf{21.3} \pm 0.4$ & $24.6 \pm 0.7$ \\ \hline
Wholesale & (440,7,1,17)  & $28.5 \pm 0.2$ & $\textbf{28.4} \pm 0.3$ & $41.6 \pm 2.4$ & $30.4 \pm 0.4$ & $45.4 \pm 1.8$ \\ \hline
Wine White \cite{cortez2009wine} & (4898,11,2,29)  & $12.4 \pm 0.1$ & $12.4 \pm 0.1$ & $12.4 \pm 0.4$ & $\textbf{8.9} \pm 0.1$ & $11.1 \pm 0.6$ \\ \hline
Wine Red \cite{cortez2009wine} & (1599,11,2,135)  & $19.8 \pm 0.1$ & $19.8 \pm 0.1$ & $18.5 \pm 0.5$ & $\textbf{12.0} \pm 0.2$ & $15.4 \pm 0.5$ \\ \hline
\end{tabular}
\caption{Average classification error rates (in percentage) with standard deviations across all real data experiments. For each experiment, the best-performing method is highlighted in bold. The number of leaf regions $m$ are based on a single decision tree on the full data.}
\label{table1}
\end{table*}

\begin{figure}[ht]
	\centering
	\includegraphics[width=0.45\textwidth,trim={0cm 0cm 0cm 0cm},clip]{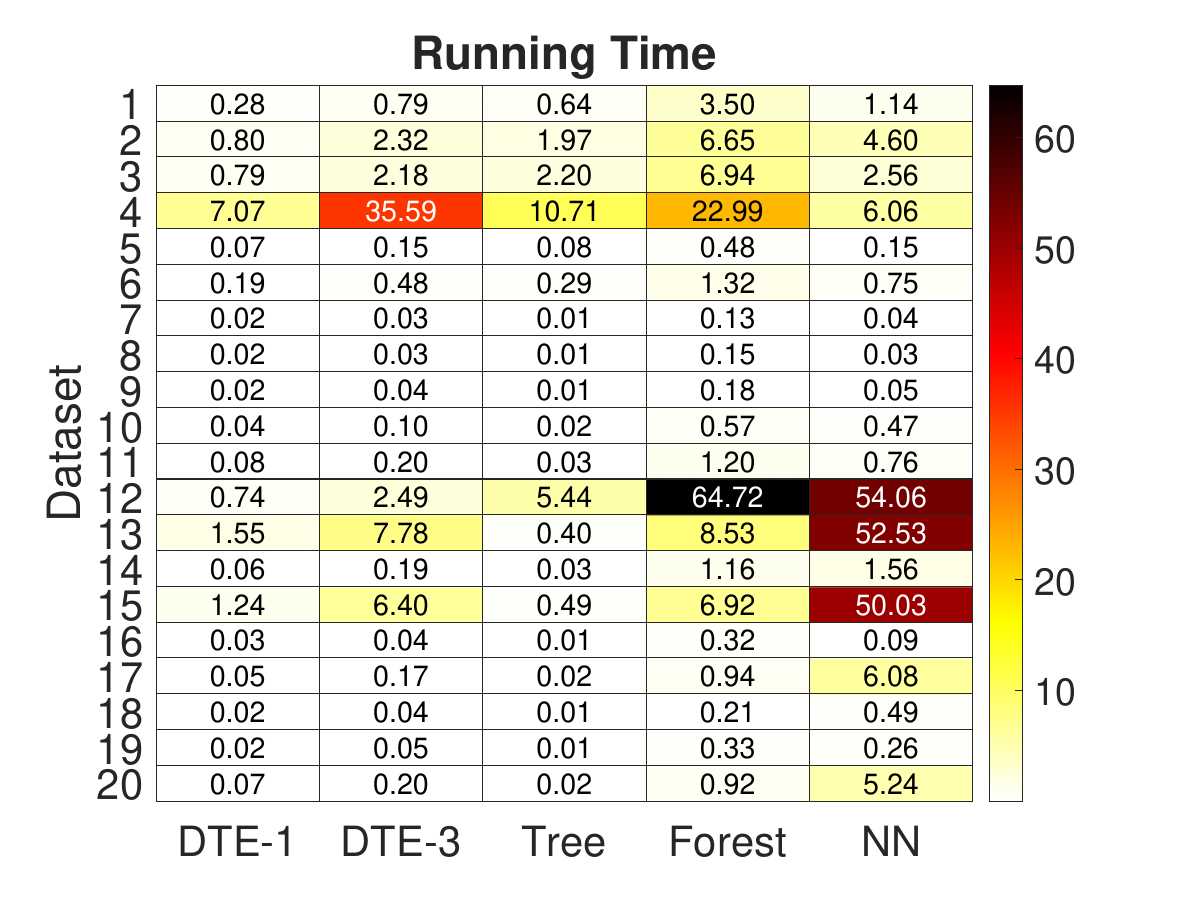}
	\caption{This figures shows the running time, in seconds, for each method and each dataset, including training and testing. We report the average running time for each dataset, averaged over all cross validation runs. 
    }
	\label{fig2}
\end{figure}

\section{Conclusion}

In this paper, we introduce the decision tree embedding, which can be viewed as a combination of neural networks and decision trees, inheriting advantages from both. The method performs a linear transformation and offers scalable inference like a neural network, while its weights are derived from tree-based leaf means, providing strong interpretability and an adaptive, data-driven choice of hidden-layer size without relying on back-propagation. We provide theoretical characterization, simulation-based visualizations illustrating its mechanism, and numerical experiments demonstrating its effectiveness.

There are several promising directions for further development. At present, classification performance may deteriorate when $t$ becomes large, as the resulting increase in dimension $m$ leads to higher estimation variance in LDA. This suggests an opportunity to introduce an additional dimension reduction layer before classification, potentially allowing the method to consistently benefit from larger ensembles, or to employ an alternative classifier that is less susceptible to increases in dimensionality. Another direction is to explore trees built using alternative criteria, such as regression trees, other impurity measures, or structured splitting rules. Finally, because the formulation closely resembles a neural network, the method could be incorporated into more complex architectures or applied to domain-specific tasks.

\bibliographystyle{IEEEtran}
\bibliography{shen,general}

\begin{thebibliography}{10}
\providecommand{\url}[1]{#1}
\csname url@samestyle\endcsname
\providecommand{\newblock}{\relax}
\providecommand{\bibinfo}[2]{#2}
\providecommand{\BIBentrySTDinterwordspacing}{\spaceskip=0pt\relax}
\providecommand{\BIBentryALTinterwordstretchfactor}{4}
\providecommand{\BIBentryALTinterwordspacing}{\spaceskip=\fontdimen2\font plus
\BIBentryALTinterwordstretchfactor\fontdimen3\font minus
  \fontdimen4\font\relax}
\providecommand{\BIBforeignlanguage}[2]{{%
\expandafter\ifx\csname l@#1\endcsname\relax
\typeout{** WARNING: IEEEtran.bst: No hyphenation pattern has been}%
\typeout{** loaded for the language `#1'. Using the pattern for}%
\typeout{** the default language instead.}%
\else
\language=\csname l@#1\endcsname
\fi
#2}}
\providecommand{\BIBdecl}{\relax}
\BIBdecl

\bibitem{quinlan1993c45}
J.~R. Quinlan, \emph{C4.5: Programs for Machine Learning}.\hskip 1em plus 0.5em
  minus 0.4em\relax San Mateo, CA: Morgan Kaufmann, 1993.

\bibitem{breiman1984cart}
L.~Breiman, J.~Friedman, R.~Olshen, and C.~Stone, \emph{Classification and
  Regression Trees}.\hskip 1em plus 0.5em minus 0.4em\relax Belmont, CA:
  Wadsworth International Group, 1984.

\bibitem{breiman1996bagging}
L.~Breiman, ``Bagging predictors,'' \emph{Machine Learning}, vol.~24, no.~2,
  pp. 123--140, 1996.

\bibitem{breiman2001random}
------, ``Random forests,'' \emph{Machine Learning}, vol.~45, no.~1, pp. 5--32,
  2001.

\bibitem{RFSparse}
T.~M. Tomita, J.~Browne, C.~Shen, J.~Chung, J.~L. Patsolic, B.~Falk, J.~Yim,
  C.~E. Priebe, R.~Burns, M.~Maggioni, and J.~T. Vogelstein, ``Sparse
  projection oblique randomer forests,'' \emph{Journal of Machine Learning
  Research}, vol.~21, no. 104, pp. 1--39, 2020.

\bibitem{rodriguez2006rotation}
J.~J. Rodriguez, L.~I. Kuncheva, and C.~J. Alonso, ``Rotation forest: A new
  classifier ensemble method,'' \emph{IEEE Transactions on Pattern Analysis and
  Machine Intelligence}, vol.~28, no.~10, pp. 1619--1630, 2006.

\bibitem{menze2011oblique}
B.~H. Menze, B.~M. Kelm, R.~Masuch, U.~Himmelreich, and F.~A. Hamprecht, ``On
  oblique random forests,'' in \emph{Proceedings of the Joint DAGM and OAGM
  Symposium}.\hskip 1em plus 0.5em minus 0.4em\relax Springer, 2011, pp.
  453--462.

\bibitem{srivastava2014dropout}
N.~Srivastava, G.~E. Hinton, A.~Krizhevsky, I.~Sutskever, and R.~Salakhutdinov,
  ``Dropout: A simple way to prevent neural networks from overfitting,''
  \emph{Journal of Machine Learning Research}, vol.~15, pp. 1929--1958, 2014.

\bibitem{he2016deep}
K.~He, X.~Zhang, S.~Ren, and J.~Sun, ``Deep residual learning for image
  recognition,'' in \emph{Proceedings of the IEEE Conference on Computer Vision
  and Pattern Recognition (CVPR)}, 2016, pp. 770--778.

\bibitem{hornik1991approximation}
K.~Hornik, ``Approximation capabilities of multilayer feedforward networks,''
  \emph{Neural Networks}, vol.~4, no.~2, pp. 251--257, 1991.

\bibitem{lecun2015deep}
Y.~LeCun, Y.~Bengio, and G.~Hinton, ``Deep learning,'' \emph{Nature}, vol. 521,
  no. 7553, pp. 436--444, 2015.

\bibitem{goodfellow2016deep}
I.~Goodfellow, Y.~Bengio, and A.~Courville, \emph{Deep Learning}.\hskip 1em
  plus 0.5em minus 0.4em\relax MIT Press, 2016.

\bibitem{lecun1998gradient}
Y.~LeCun, L.~Bottou, Y.~Bengio, and P.~Haffner, ``Gradient-based learning
  applied to document recognition,'' \emph{Proceedings of the IEEE}, vol.~86,
  no.~11, pp. 2278--2324, 1998.

\bibitem{vaswani2017attention}
A.~Vaswani, N.~Shazeer, N.~Parmar, J.~Uszkoreit, L.~Jones, A.~N. Gomez,
  L.~Kaiser, and I.~Polosukhin, ``Attention is all you need,'' in
  \emph{Advances in Neural Information Processing Systems (NeurIPS)}, 2017, pp.
  5998--6008.

\bibitem{caruana2006empirical}
R.~Caruana and A.~Niculescu-Mizil, ``An empirical comparison of supervised
  learning algorithms,'' in \emph{Proceedings of the 23rd International
  Conference on Machine Learning (ICML)}.\hskip 1em plus 0.5em minus
  0.4em\relax ACM, 2006, pp. 161--168.

\bibitem{fernandez2014we}
M.~Fernández-Delgado, E.~Cernadas, S.~Barro, and D.~Amorim, ``Do we need
  hundreds of classifiers to solve real world classification problems?''
  \emph{Journal of Machine Learning Research}, vol.~15, no.~1, pp. 3133--3181,
  2014.

\bibitem{grinsztajn2022why}
L.~Grinsztajn, E.~Oyallon, and G.~Varoquaux, ``Why do tree-based models still
  outperform deep learning on typical tabular data?'' in \emph{Advances in
  Neural Information Processing Systems (NeurIPS)}, 2022.

\bibitem{ferreira2021computational}
D.~R. Ferreira, A.~E. Lazzaretti, and A.~S. Ferreira~Junior, ``Computational
  and memory cost analysis of ensemble tree classifiers,'' \emph{IEEE Access},
  vol.~9, pp. 121--136, 2021.

\bibitem{He2014}
X.~He, J.~Pan, O.~Jin, T.~Xu, B.~Liu, T.~Xu, Y.~Shi, A.~Atallah, R.~Herbrich,
  S.~Bowers, and J.~Q. Candela, ``Practical lessons from predicting clicks on
  ads at facebook,'' in \emph{Proceedings of the 20th ACM Conference on
  Knowledge Discovery and Data Mining (KDD)}, 2014.

\bibitem{Kontschieder2015}
P.~Kontschieder, M.~Fiterau, A.~Criminisi, and S.~Rota~Bulo, ``Deep neural
  decision forests,'' in \emph{Proceedings of the IEEE International Conference
  on Computer Vision (ICCV)}, 2015.

\bibitem{wang2016rf2nn}
S.~Wang, J.~Tang, and H.~Liu, ``Using a random forest to inspire a neural
  network and improving on it,'' in \emph{IEEE International Conference on Data
  Mining Workshops (ICDMW)}.\hskip 1em plus 0.5em minus 0.4em\relax IEEE, 2016,
  pp. 1049--1056.

\bibitem{biau2019neural}
G.~Biau, E.~Scornet, and J.~Welbl, ``Neural random forests,'' \emph{Statistical
  Analysis and Data Mining: The ASA Data Science Journal}, vol.~12, no.~5, pp.
  397--407, 2019.

\bibitem{togunwa2023deep}
O.~Togunwa, J.~Afolabi, and S.~Fasina, ``A deep hybrid model for maternal
  health risk classification,'' \emph{Frontiers in Artificial Intelligence},
  vol.~6, p. 1213436, 2023.

\bibitem{konstantinov2023attention}
N.~Konstantinov, D.~Grigoryev, and A.~Babenko, ``Neural attention forests,''
  \emph{arXiv preprint arXiv:2304.05980}, 2023.

\bibitem{qiu2024nnrf}
Y.~Qiu, M.~Xu, and B.~Yu, ``Neural networks meet random forests,''
  \emph{Journal of the Royal Statistical Society: Series B (Statistical
  Methodology)}, vol.~86, no.~5, pp. 1435--1460, 2024.

\bibitem{nickzamir2025hybrid}
S.~Nickzamir and S.~A. Gandab, ``Hybrid random forest and convolutional neural
  network framework for hyperspectral image classification,'' \emph{arXiv
  preprint arXiv:2502.00232}, 2025.

\bibitem{NNGraphSufficiency}
C.~Shen and Y.~Dong, ``A graph sufficiency perspective for neural networks,''
  \emph{arXiv preprint arXiv:2507.10215}, 2026.

\bibitem{Kelly2023UCI}
M.~Kelly, R.~Longjohn, and K.~Nottingham, ``The uci machine learning
  repository,'' \url{https://archive.ics.uci.edu}, 2023.

\bibitem{mccallum2000automating}
A.~K. McCallum, K.~Nigam, J.~Rennie, and K.~Seymore, ``Automating the
  construction of internet portals with machine learning,'' \emph{Information
  Retrieval}, vol.~3, pp. 127--163, 2000.

\bibitem{giles1998citeseer}
C.~L. Giles, K.~D. Bollacker, and S.~Lawrence, ``Citeseer: An automatic
  citation indexing system,'' in \emph{Proceedings of the Third ACM Conference
  on Digital Libraries}, 1998, pp. 89--98.

\bibitem{Fanty1991}
M.~Fanty and R.~Cole, ``Spoken letter recognition,'' in \emph{Advances in
  Neural Information Processing Systems}, vol.~3, 1991.

\bibitem{SimBakerBsat2003}
T.~Sim, S.~Baker, and M.~Bsat, ``The {CMU} pose, illumination, and expression
  database,'' \emph{IEEE Transactions on Pattern Analysis and Machine
  Intelligence}, vol.~25, no.~12, pp. 1615--1618, 2003.

\bibitem{Fisherfaces}
P.~N. Belhumeur, J.~P. Hespanha, and D.~Kriegman, ``Eigenfaces vs. fisherfaces:
  Recognition using class specific linear projection,'' \emph{Pattern Analysis
  and Machine Intelligence, IEEE Transactions on}, vol.~19, no.~7, pp.
  711--720, 1997.

\bibitem{fisher1936use}
R.~A. Fisher, ``The use of multiple measurements in taxonomic problems,''
  \emph{Annals of Eugenics}, vol.~7, no.~2, pp. 179--188, 1936.

\bibitem{AeberHard1994}
S.~Aeberhard, D.~Coomans, and O.~de~Vel, ``Comparative analysis of statistical
  pattern recognition methods in high dimensional settings,'' \emph{Pattern
  Recognition}, vol.~27, no.~8, pp. 1065--1077, 1994.

\bibitem{street1993nuclear}
W.~Street, W.~Wolberg, and O.~Mangasarian, ``Nuclear feature extraction for
  breast tumor diagnosis,'' in \emph{Proceedings of SPIE - The International
  Society for Optical Engineering}, 1993.

\bibitem{alon1999broad}
U.~Alon, N.~Barkai, D.~A. Notterman, K.~Gish, S.~Ybarra, D.~Mack, and A.~J.
  Levine, ``Broad patterns of gene expression revealed by clustering analysis
  of tumor and normal colon tissues probed by oligonucleotide arrays,''
  \emph{Proceedings of the National Academy of Sciences}, vol.~96, no.~12, pp.
  6745--6750, 1999.

\bibitem{golub1999molecular}
T.~R. Golub, D.~K. Slonim, P.~Tamayo, C.~Huard, M.~Gaasenbeek, J.~P. Mesirov,
  H.~Coller, M.~L. Loh, J.~R. Downing, M.~A. Caligiuri, C.~D. Bloomfield, and
  E.~S. Lander, ``Molecular classification of cancer: class discovery and class
  prediction by gene expression monitoring,'' \emph{Science}, vol. 286, no.
  5439, pp. 531--537, 1999.

\bibitem{DalPozzolo2018CreditCardFraud}
A.~Dal~Pozzolo, G.~Boracchi, O.~Caelen, C.~Alippi, and G.~Bontempi, ``Credit
  card fraud detection: a realistic modeling and a novel learning strategy,''
  \emph{IEEE Transactions on Neural Networks and Learning Systems}, vol.~29,
  no.~8, pp. 3784--3797, 2018.

\bibitem{moro2011bankmarketing}
S.~Moro, R.~Laureano, and P.~Cortez, ``Using data mining for bank direct
  marketing: An application of the {CRISP-DM} methodology,'' in
  \emph{Proceedings of the European Simulation and Modelling Conference}, 2011,
  pp. 117--121.

\bibitem{martins2021early}
M.~V. Martins, D.~Tolledo, J.~Machado, L.~M.~T. Baptista, and V.~Realinho,
  ``Early prediction of student’s performance in higher education: A case
  study,'' in \emph{Trends and Applications in Information Systems and
  Technologies}, vol. 1365.\hskip 1em plus 0.5em minus 0.4em\relax Springer,
  2021.

\bibitem{kohavi-nbtree}
R.~Kohavi, ``Scaling up the accuracy of naive-bayes classifiers: a
  decision-tree hybrid,'' in \emph{Proceedings of the Second International
  Conference on Knowledge Discovery and Data Mining}, 1996, pp. 202--207.

\bibitem{cortez2009wine}
P.~Cortez, A.~Cerdeira, F.~Almeida, T.~Matos, and J.~Reis, ``Modeling wine
  preferences by data mining from physicochemical properties,'' \emph{Decision
  Support Systems}, vol.~47, no.~4, pp. 547--553, 2009.

\end{thebibliography}

\clearpage
\onecolumn
\setcounter{figure}{0}
\renewcommand{\thealgorithm}{C\arabic{algorithm}}
\renewcommand{\thefigure}{E\arabic{figure}}
\renewcommand{\thesubsection}{\thesection.\arabic{subsection}}
\renewcommand{\thesubsubsection}{\thesubsection.\arabic{subsubsection}}
\pagenumbering{arabic}

\bigskip
\begin{center}
{\large\bf APPENDIX}
\end{center}
\section{Proofs}
\label{sec:proofs}

\thmOne*
\begin{proof}
Let $\{\mathcal{R}_1,\ldots,\mathcal{R}_m\}$ be the leaf partition of $\mathcal{X}$. The map $X \mapsto \mathcal{R}(X) \in \{1,\ldots,m\}$ is measurable with respect to the $\sigma$–field generated by the partition, i.e.
\[
\sigma(\mathcal{R}(X)) = \sigma\big(\{X\in\mathcal{R}_j\}_{j=1}^m\big).
\]
By definition of the decision tree embedding, $Z$ is a measurable function of
$\mathcal{R}(X)$; hence the generated $\sigma$–fields coincide:
\[
\sigma(Z) = \sigma(\mathcal{R}(X)).
\]
Therefore,
\begin{align}
\label{eq1}
P(Y\mid Z) = P(Y\mid \mathcal{R}(X)).
\end{align}
For each $j$, define
\[
P(Y=c\mid X\in\mathcal{R}_j)
:= \frac{P(Y=c,\,X\in\mathcal{R}_j)}{P(X\in\mathcal{R}_j)}.
\]
Bayes $\varepsilon$-homogeneity means that for all $x,x'\in\mathcal{R}_j$,
\[
\|P(Y\mid X=x)-P(Y\mid X=x')\|_1 \le \varepsilon_j,
\]
and define the global bound $\varepsilon = \max_j \varepsilon_j$.

Fix any $x\in\mathcal{R}_j$.  
Since $P(Y\mid\mathcal{R}_j)$ is the average of $P(Y\mid X=x')$ over $x'\in \mathcal{R}_j$, convexity of the $L^1$ norm implies
\[
\|P(Y\mid X=x)-P(Y\mid \mathcal{R}_j)\|_1
\le \varepsilon_j \le \varepsilon.
\]
Using Equation~\ref{eq1} and that $\mathcal{R}(X)=j$ if and only if $X\in \mathcal{R}_j$, we have almost surely
\[
\|P(Y\mid X)-P(Y\mid Z)\|_1
= \|P(Y\mid X)-P(Y\mid \mathcal{R}(X))\|_1
\le \varepsilon.
\]

When $\varepsilon=0$, the conditional distribution $P(Y\mid X=x)$ is constant on each $\mathcal{R}_j$, and thus $P(Y\mid X)=P(Y\mid Z)$ almost surely. Hence $Z$ is a sufficient statistic for $Y$ relative to $X$.
\end{proof}

\thmTwo*
\begin{proof}
Fix any $x\in\mathcal{R}_j$. Its DTE embedding is given by
\[
Z_{j'}(x) = x^\top \mu_{j'} - \|\mu_{j'}\|^2 , \quad j'=1,\ldots,m.
\]

We first show that the $j$th coordinate is the maximal one. Using the identity
\[
x^\top \mu_{j'} - \|\mu_{j'}\|^2 = -\tfrac12 \|x-\mu_{j'}\|^2 + \tfrac12\|x\|^2,
\]
we obtain
\begin{align}
\label{eq2}
Z_j(x) - Z_{j'}(x) = \tfrac12\bigl(\|x-\mu_{j'}\|^2 - \|x-\mu_j\|^2\bigr).
\end{align}

Since $\mu_j$ is the mean of the leaf $\mathcal{R}_j$, all points $x\in\mathcal{R}_j$ satisfy
\begin{align}
\label{eq3}
\|x-\mu_j\|^2 \le \|x-\mu_{j'}\|^2, \quad\forall j'\neq j,
\end{align}
because each leaf of a decision tree corresponds to the region of points split so as to be closest (in squared Euclidean distance) to its own mean. Combining Equation~\ref{eq2} and ~\ref{eq3} yields
\begin{align}
\label{eq4}
Z_j(x) \ge Z_{j'}(x) \quad\forall j'\neq j.
\end{align}

Now consider the indicator classifier
\begin{align}
g(Z)=\arg\max_{c\in[K]} \gamma_c^\top Z, 
\quad (\gamma_c)_j = 1_{\{j\in\mathcal{J}_c\}}.
\end{align}
Since $\gamma_c^\top Z(x)$ selects the maximum among coordinates  belonging to class $c$, and Equation~\ref{eq4} shows that the maximal coordinate of $Z(x)$ is the $j$th one, it follows that
\begin{align*}
g(Z(x)) = c_j,
\end{align*}
where $c_j$ is the majority class in region $\mathcal{R}_j$.

Thus, $g$ predicts a constant label on each leaf region. The misclassification probability is therefore
\begin{align}
\label{eq5}
\mathbb{P}\bigl(g(Z(X))\ne Y\bigr) = \sum_{j=1}^m P(X\in\mathcal{R}_j)\,P(Y\ne c_j\mid X\in\mathcal{R}_j).
\end{align}

By definition of the impurity
\[
l_j 
= 1- \max_{c} P(Y=c\mid X\in\mathcal{R}_j) 
= P(Y\ne c_j\mid X\in\mathcal{R}_j),
\]
and substituting into Equation~\ref{eq5} gives
\[
L_g = \sum_{j=1}^m P(X\in\mathcal{R}_j)\,l_j.
\]

If each leaf is pure, then $l_j=0$ for all $j$, so $L_g=0$, and the DTE embedding is perfectly linearly separable under the indicator classifier.
\end{proof}

\end{document}